\documentclass{article}
\usepackage[preprint]{neurips_2026}
\usepackage{natbib}
\setcitestyle{authoryear,open={(},close={)}} 
\usepackage{graphicx} 
\usepackage{subcaption} 

\usepackage{verbatim} 
\usepackage{xcolor}
\usepackage{bbold}
\usepackage{hyperref}
\usepackage{geometry}
\usepackage{amsmath, amsthm, amssymb} 
\usepackage{multirow}  
\usepackage{booktabs}  
\usepackage{algorithm}
\usepackage{algpseudocode}

\DeclareMathSymbol{\shortminus}{\mathbin}{AMSa}{"39}

\usepackage{mdframed}
\newenvironment{systemprompt}[1][]{%
    \begin{mdframed}[linewidth=0.8pt, innerleftmargin=8pt, innerrightmargin=8pt, nobreak=true,
        frametitle={#1}]
    \small\ttfamily
}{%
    \end{mdframed}
}

\newtheorem{theorem}{Theorem}
\newtheorem{lemma}{Lemma}
\newcommand{\todo}[1]{\textcolor{red}{#1}}

\newcommand{\snJAT}[1]{\textcolor{teal}{${\leftarrow}\hspace{-3pt}{\bullet}$}\marginpar{\textcolor{teal}{${\leftarrow}\hspace{-3pt}{\bullet}$}{\textcolor{teal}{  {\scriptsize{JAT: #1}}}}}}

\newcommand{\snJHO}[1]{\textcolor{teal}{${\leftarrow}\hspace{-3pt}{\bullet}$}\marginpar{\textcolor{teal}{${\leftarrow}\hspace{-3pt}{\bullet}$}{\textcolor{teal}{  {\scriptsize{JHO: #1}}}}}}

\newcommand{\snBH}[1]{\textcolor{teal}{${\leftarrow}\hspace{-3pt}{\bullet}$}\marginpar{\textcolor{teal}{${\leftarrow}\hspace{-3pt}{\bullet}$}{\textcolor{teal}{  {\scriptsize{BH: #1}}}}}}

\title{Teaching and Learning under Deductive Errors}
\author{%
  Jan Arne Telle \\
  Department of Informatics\\
  University of Bergen \\
  Norway\\
 \\
  \And
  Brigt Håvardstun \\
  Department of Informatics \\
  University of Bergen \\
  Norway\\
   \AND
 Jose Hernandez-Orallo \\
 VRAIN -  Universitat Politecnica de Valencia \\
  Spain \\
 Leverhulme Centre for the Future of Intelligence - University of Cambridge \\
 United Kingdom \\
}

\begin{document}

\maketitle

\begin{abstract}\footnote{This research is part of the MT4XAI project financed by the Research Council of Norway}
Most models of machine teaching and learning assume the learner makes no errors in its internal deductive inference. However, humans and large language models in few-shot learning regimes are two important examples of learners where this does not hold. They fail on some consistency checks, and they can fail stochastically. In this paper we introduce a teaching and learning framework that takes these deductive errors into account. We specifically study the case of machine teaching, as different characterizations of the teacher can account for both machine teaching and learning. In an overhauled Probably Approximately Correct (PAC) setting, we study theoretically that, for some estimated error level, the teacher must find a PAC teaching set that with high probability will lead the learner to guess a hypothesis that is approximately correct. We study the computational complexity of six different problems related to computing optimal PAC teaching sets. We give XP algorithms parametrized by size of teaching set, with tight runtime bounds under standard complexity assumptions like ETH. These results are complemented with a small experimental study of which teaching and learning protocols can best represent the observed behavior in some LLM teaching sessions.
\end{abstract}

\section{Introduction}

Learning mostly involves inductive inference, i.e., the process of going from facts to rules, from data to models. Many different paradigms for machine learning have been introduced according to how data is presented and how hypotheses are represented, and learning algorithms have been devised to make this process effective and efficient. However, most of these paradigms assume that the learner performs perfect deductive inference, i.e., the process of going from rules to facts, from models to data. In other words, learners are assumed to be consistent and complete when checking the alternative hypotheses during the learning process. Given a concept $c$ in the hypothesis space $C$, and an example $x$, the learner invokes a consistency function $c(x)$ that returns true if the example belongs to concept $c$ and false otherwise. We are not assuming that $c$ is the correct hypothesis --that's the inductive inference search--, but we assume that if $c$ were correct then the learner would be able to classify all examples perfectly with it. 
This assumption has seemed reasonable for learners that are implemented on computers with clear algorithms and representation languages whose deductive inference mechanisms are consistent and complete, e.g., applying a logistic regression to a new input, a neural network to a new image or a decision tree to a feature vector. 

Following this tradition, machine {\em teaching}, where examples are chosen by the teacher, rather than randomly sampled from a distribution, has also assumed that the learner is deductively consistent and complete \citep{goldman1995complexity,balbach2008measuring,zilles2011models,gao2017preference,no-clash}. 
There are only very few examples in the literature where the representation is chosen in such a way that checking consistency becomes intractable or even incomputable (and hence the function cannot be complete) making approximations necessary \citep{ferri2026redundancy}.

However, humans and pretrained large language models are two important examples of learners that should not be modelled with a perfect consistency function $c$. They fail on some consistency checks, and they can fail stochastically. This is crucial for understanding few-shot learning and teaching with them, which is  particularly relevant in explainable AI, as it is generally assumed that errors will originate from inductive inference but not deductive inference~\citep{yang2021mitigating,Hvardstun2023XAIWM}. For these learners, we should consider they possess approximate consistency functions $\hat{c}_k(x_i)$ for each concept, whose error can be quantified by an error level $\gamma_L$ as:

\[ \gamma_L(c_k, x_i) = P[ c_k(x_i) \neq \hat{c}_k(x_i) ] \]

It is natural to consider that this error function depends on both the learner $L$, the example $x_i$ and the concept $c_k$. For instance, it is easier to make a mistake when checking whether 8250715087503 is prime than whether it is even. On the other hand, it is also easier to make a mistake when checking whether 8250715087503 is prime than whether 27 is a prime. 

Under this new perspective, traditional accounts of learning and teaching would simply not work. In particular, some of the simplest learning protocols, where the learner is given examples one after another and discards all the concepts that are inconsistent with the given evidence, until only one remains, would often fail catastrophically: a mistake in the consistency check at an earlier stage would rule out the right concept forever. Instead, we could keep all concepts and update the consistency vector for each concept as the process progresses, but would this lead to the identification of the right concept?\footnote{Note  the learner can differ on $\hat{c}_k$ in different occasions, since $\gamma_L$ is stochastic. Here we are not assuming that the concepts are stochastic or there is some inherent noise in the world, but the checking of the consistency of concepts is stochastic.}

We claim that the right approach for exploring the learnability and teachability in this paradigm is an adaptation of the Probably Approximately Correct (PAC) setting. Note that in classical PAC Learning \citep{valiant1984theory,valiant2013probably} the stochasticity comes from sampling of examples, whereas our focus is on the stochasticity arising from errors in consistency checking.
Combining these two sources of stochasticity does not pose any fundamental problems, but it does complicate the presentation.
Therefore, we will in this paper assume the setting of Machine Teaching \cite{zhu2018overview}, where the examples given to the learner are chosen by a deterministic and consistent teacher rather than sampled, to highlight the new findings in this field that we will call PAC Teaching.

Because of the errors in consistency checking, the learner will never know for certain that the concept $c$ being selected is the correct target concept $c^*$, not (only) because of the inductive inference problem, but because of a deductive inference fallibility, represented by the error level $\gamma_L$. It is then more natural to consider that the selected $c$ can only be {\em approximately correct} even in cases of small hypothesis classes and a sufficient large number of examples. Also, because the error is of stochastic character, it also makes more sense to consider how {\em probable} it is that an approximately correct concept can be attained. The question becomes: given a learner  $L$, characterized by a deductive error function $\gamma_L$ and a particular learning protocol, can we devise teaching algorithms that lead to effective PAC teaching? In this paper we show that the answer is yes. After first defining the PAC teaching framework formally, and arguing for its suitability, we show in Section \ref{sec:betteralgo} that the problem of computing optimal PAC teaching sets, which comes in six different variants, can be solved by algorithms whose asymptotic runtimes are tight under standard complexity assumptions.

The rest of the paper is organized as follows.
In Section 2 we define PAC Teaching formally. In Section \ref{sec:ill} we argue that the PAC teaching framework is useful whenever the teaching success is inversely correlated with deductive errors, and show that this holds in experiments with LLMs as learners. In Section \ref{sec:LandT} we discuss various assumptions on learner and teacher behavior, and give results showing the advantage of taking the imperfect consistency checks into account. In Section \ref{sec:betteralgo} we give algorithms and tight complexity bounds  on the problems of computing optimal PAC teaching sets, and we also give some heuristic algorithms for computing PAC teaching sets.

\section{Definitions}

As in classical Machine Teaching, we have a binary consistency matrix between an example set $X$ and a concept set $C$, and denote by $c(x)$ the class label 1 or 0 (In or Out, True or False) of example $x \in X$ for concept $c \in C$. An example $x$ labelled $b \in \{0,1\}$, given as $(x,b)$, is consistent with $c$ if $c(x)=b$. We have a target concept, and the goal in classical machine teaching is to find a small set of labelled examples, the so-called teaching set $S$, such that the only concept that is consistent with all of $S$ is the target concept.\footnote{As noted above, we could also allow the teaching set to be randomly sampled, for a better alignment with formal models of Machine Learning, at the expense of introducing another source of stochasticity.}

To account for learners that are not perfect when doing consistency checks we introduce error probabilities  $\gamma_L:C \times X \rightarrow [0,1]$ where $\gamma_L(c,x)$ is the probability of learner $L$ making an error when checking the value of $c(x)$. 
In the setting we introduce, when the learner $L$ is given example $x$ labelled $b$ the error probability $\gamma_L(c,x)$ is used to decide if concept $c$ is \emph{L-consistent} with $(x,b)$ or not. If indeed $c(x)=b$ then the learner concludes with probability $1-\gamma_L(c,x)$ that $c$ is L-consistent with this labelled example (and with probability $\gamma_L(c,x)$ that $c$ is not L-consistent), while if $c(x) \neq b$ 
we have the opposite (the learner concludes with probability $\gamma_L(c,x)$ that $c$ is L-consistent and probability $1-\gamma_L(c,x)$ that $c$ is not L-consistent). This imperfect and stochastic version of $c$ that $L$ has will be denoted as $\hat{c}_L$ or simply $\hat{c}$. Note these errors can vary for each example and concept pair, but are symmetric with respect to checking for a correct or incorrect label, as the source of error is only tied to checking the value of $c(x)$. 

For a concrete case, consider a concept class of subsets of natural numbers and teaching target concept $c^* = {\tt ending\texttt{-}in\texttt{-}7}$ using the example $x=27$ and checking consistency with $c = {\tt primes}$, a pair for which $\gamma_L(c,x)=0.1$. The learner is given $(x=27,b=1)$ and even though $c(x)=0$ so $c$ is not consistent there is a $10\%$ chance that the learner will conclude that $c$ is L-consistent, because of the error involved in $\hat{c}(x)$, in effect mistaking 27 as being prime, and if so reinforcing the belief that $c$ could be a good concept for the unknown $c^*$, when it is not. A teacher aware of this could instead have chosen an example of lower error to rule out $c$, say $(x=13,b=0)$ with $\gamma_L(c,13)=0.01$.

Informally, in PAC teaching, when using a good teaching set $S$ for the target concept $c^*$, it should be {\bf probable} that the concept $c$ chosen by the learner (e.g., the concept that is L-consistent with most labelled examples in $S$), is {\bf approximately} equal to the target $c^*$.
Note that we have two distinct cases, that we may call teaching for {\bf identification} or teaching for {\bf employment}. When the goal is identification then we want the chosen concept $c$, which is {\em named} by the learner and used by a perfect deductive user (e.g., the teacher), to be consistent with the target $c^*$ on as many examples as possible, i.e. with high expected value of the event $c(x)=c^*(x)$ for some distribution over examples, as measured by the following {\em similarity} function
\begin{equation*}
  sim(c,c^*)  
  = \mathbb{E}_{x \sim {\mathcal D}(X)} [ P({c}(x) = c^*(x))] \hspace{0.2cm}
  ( = 1/|X|\textstyle{\sum_{x \in X}}{\mathbb{1}[c(x) = c^*(x)]} \mbox{if uniform $\mathcal D$, finite $X$})
    \label{eq:simexp}
\end{equation*}

However, when the goal is employment by the learners themselves, then, based on the reasonable assumption that the same level of errors the learner $L$ makes in consistency checking occur not only during teaching but also in employment,
 we instead define a similarity function $sim_L$ as follows

\begin{equation*}
  sim_L(c,c^*)  
  = \mathbb{E}_{x \sim {\mathcal D}(X)} [ P(\hat{c}(x) = c^*(x))] 
    \label{eq:simexpL}
\end{equation*}

 In other words, measuring the probability of $c$ being L-consistent with $x$ when labelled by $c^*(x)$. 
If we expand using the connection between $\hat{c}(x)$ and $\gamma_L(c,x)$, and assume that the distribution over examples is uniform and the example set is finite, then we get the following

\begin{equation*}
  sim_L(c,c^*)   
  = \frac{1}{|X|}\sum_{x \in X}{\begin{cases}
        1-\gamma_L(c,x)\> \mbox{ if } c(x) = c^*(x)\\
        \gamma_L(c,x) \qquad \mbox{ if } c (x) \neq c^*(x)
    \end{cases}}
     \label{eq:simL}
\end{equation*}

In PAC teaching for identification, for a given target concept $c^*$, probability $p$, and approximation ratio $q$, the goal is to find a small teaching set $S$ such that the concept $L(S)$ chosen by the learner $L$ when given $S$ satisfies Equation \ref{eq:PACTeaching} (in PAC teaching for employment replace $sim$ by $sim_L$)
\begin{equation}
    P[sim(L(S), c^*)\geq q] \geq p
    \label{eq:PACTeaching}
\end{equation}

For identification teaching with approximation ratio $q$ we define the {\em good} subset  of concepts to be $G(c^*,q)= \{c \in C: sim(c,c^*) \geq q\}$, and say that these are $q$-similar to $c^*$ (for employment replace $sim$
 by $sim_L$ to get $G_L(c^*,q)$ that are $q_L$-similar to $c^*$). The Probability of $q$-Approximately Correct Teaching for identification (or employment) when using teaching set $S$ is then the probability that $L(S) \in G(c^*,q)$  (or $L(S) \in G_L(c^*,q)$).  This depends on the behavior of the learner $L$ which will be discussed in Section \ref{sec:LandT}.
However, it should be clear that higher errors $\gamma_L$ will lead to larger teaching sets for fixed $p$ and $q$, as illustrated in Figure
\ref{fig:DualImages}.

\begin{figure}[!htbp]
    \centering
    \includegraphics[width=0.80\linewidth]{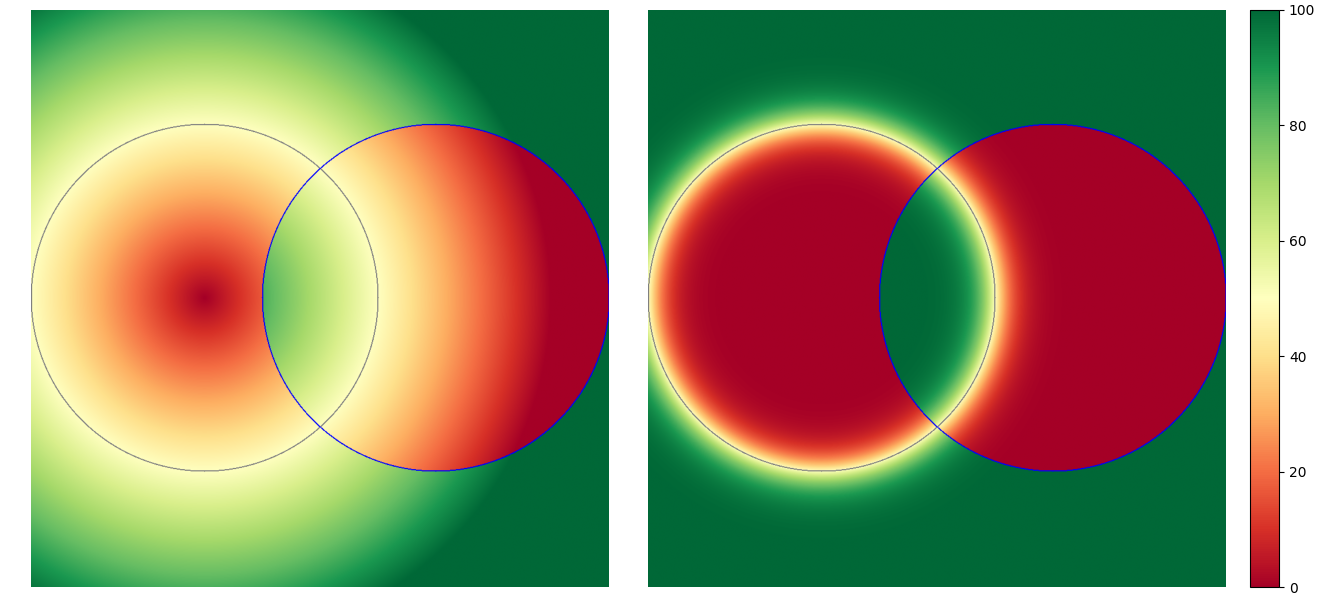}
    \caption{For fixed $p,q$ higher errors give larger PAC teaching sets. In both figures above the concept class consists of circles in the plane with $c(x)=1$ when point $x$ is within circle $c$, with only two such concepts shown: target $c^*$ the circle on the right, and $c'$ the circle on the left. The color of a point $x$ gives the probability of $c'$ being L-consistent with $x$ when labelled by $c^*(x)$, from high probability in dark green to low probability in dark red. Thus, the more green and less red the higher the L-similarity between $c'$ and the target.
    For the figure on the right we have non-zero error $\gamma_L(c',x)>0$ only when $x$ is very close to perimeter of $c'$. 
    The symmetric difference is then dark red, apart from the perimeter of $c'$ where the error kicks in. 
         In the left figure  $\gamma_L(c',x)$ is higher, and proportional to the distance from  $x$ to perimeter of $c'$, and the teacher has less examples to choose from to rule out $c'$ with high probability.
         If many circles had this same error then more teaching examples would be needed to arrive with high probability at a circle that closely approximates the target (small symmetric difference), as any single example would be dark red for only few circles.
         }
    \label{fig:DualImages}
\end{figure}

\section{Applications of PAC Teaching}\label{sec:ill}

In this section we argue that (deductive) errors in consistency checking occur with many (inductive) learners, in particular with LLMs, and that the PAC teaching framework is useful whenever the success of teaching is correlated with the deductive errors.

As discussed in the introduction, the learning process involves at least two  types of errors on part of the learner, that we may call the deductive error (inference involving consistency checks) and the inductive error (inference in searching for the right hypothesis). What we may call the total error rate of the learning process depends on both types of errors, and possibly other types, for example the error related to understanding correctly the learning task. The new PAC teaching framework assumes that the deductive errors form an important part of the total whose influence is significant. 
In many cases this is clearly the case. What about using LLMs as the learner? 
Recent research on LLMs show reliability fluctuations that deviate from what we would expect given human criteria of difficulty, see e.g. \cite{zhou2024larger}.
Can deductive errors be derived from LLMs, and are they amenable to the PAC teaching framework?  To answer these questions we perform the following experiments:
\begin{itemize}
    \item Derive errors of the LLM on some fixed consistency checks, and call these deductive errors.
    \item Extract total success of the LLM as learner over varying teaching sets, on a task involving the same consistency checks, in scenarios believed to have both low and high total error
    \item For given teaching sets, do we see a negative correlation between total success and the deductive error, and is the variation in this correlation explainable by varying total error?
\end{itemize}

A positive answer to the final question would indicate that deductive errors can be derived from LLMs, and that the utility of the PAC teaching framework can also hold when we use LLMs as learners.

For experiments  to answer these questions, we consider a concept class over positive integers.\footnote{Code available at: \url{https://github.com/BrigtHaavardstun/PAC_teaching}} Let $C=\{c_{5},c_{7},c_{11},c_{13},c_{17}\}$  and $X=\{1\cdots 1000\}$, with 
$c_k(x)=1$ iff $x$ is a multiple of $k$. Note that for some prime numbers (2, 3, 5) there are easy algorithms deciding its multiples in base 10, and for this reason we included only one of them, namely 5, together with the next four prime numbers.
We use GPT-5-nano (gpt-5-nano-2025-08-07) \cite{singh2026openaigpt5card}  with reasoning effort set to minimal. 
Stronger models achieve near-perfect deductions on consistency  for this concept class, leaving insufficient signal for illustrative purposes. GPT-5-nano falls in a capability category where it is strong enough to follow the format of inductive teaching but not so strong that checking for divisors becomes trivial. We use the default sampling parameters (temperature = 1, top\_p = 1) to avoid introducing additional constraints on the model's output distribution. 



We first describe how we measure the deductive error. In the selected concept class, consistency checking equates to checking `Is $k$ a divisor of $x$', for a given $c_k$ and $x$. Hence, this is our user prompt. See the Appendix for precise System Prompts and User Messages used for all experiments, or see the Github repository.
Each pair of $x \in X$ and $k$ (for $c_k \in C$) is queried 20 times, and we calculate the deductive error $\gamma_L(c_k,x)$ as the fraction of times the LLM answered incorrectly (including nonsensical answers that form $0.3\%$ overall\footnote{Examples of nonsensical answers: Answer = None(...) , Answer = Benjamin(...), Answer =  Delete this line(...) .}). 

For the case of small total teaching error, we limit the experiment to teaching sets on a single positive example. In other words, define $X'= \{x \in X: \mbox{exactly one of 5,7,11,13,17 is a divisor of $x$}\}$ and use only teaching sets of type $\{x\}$ for $x \in X'$ such that the information uniquely identifies one concept. 
For each $x \in X'$, we query the LLM 100 times with the prompt 'Which of 5, 7, 11, 13 and 17 is a divisor of \{x\}?' and call the fraction of incorrect answers the {\em total teaching error} for this teaching set $\{x\}$. 

From all this error data we get the plot in Figure \ref{fig:llm_experiment} (a), where each point is a single $x$ value, whose color and shape determines which $c_k \in C$ is the target (i.e. which unique $k$ that divides $x$). The horizontal axis is the mean deductive error $1/5 \sum_{c_k \in C} \gamma_L(c_k,x)$ for $x$, and the vertical axis is the total teaching error for $\{x\}$.

\begin{figure}[ht]
    \centering
    \begin{subfigure}[b]{0.45\textwidth}
        \centering
        \includegraphics[height=0.27\textheight]{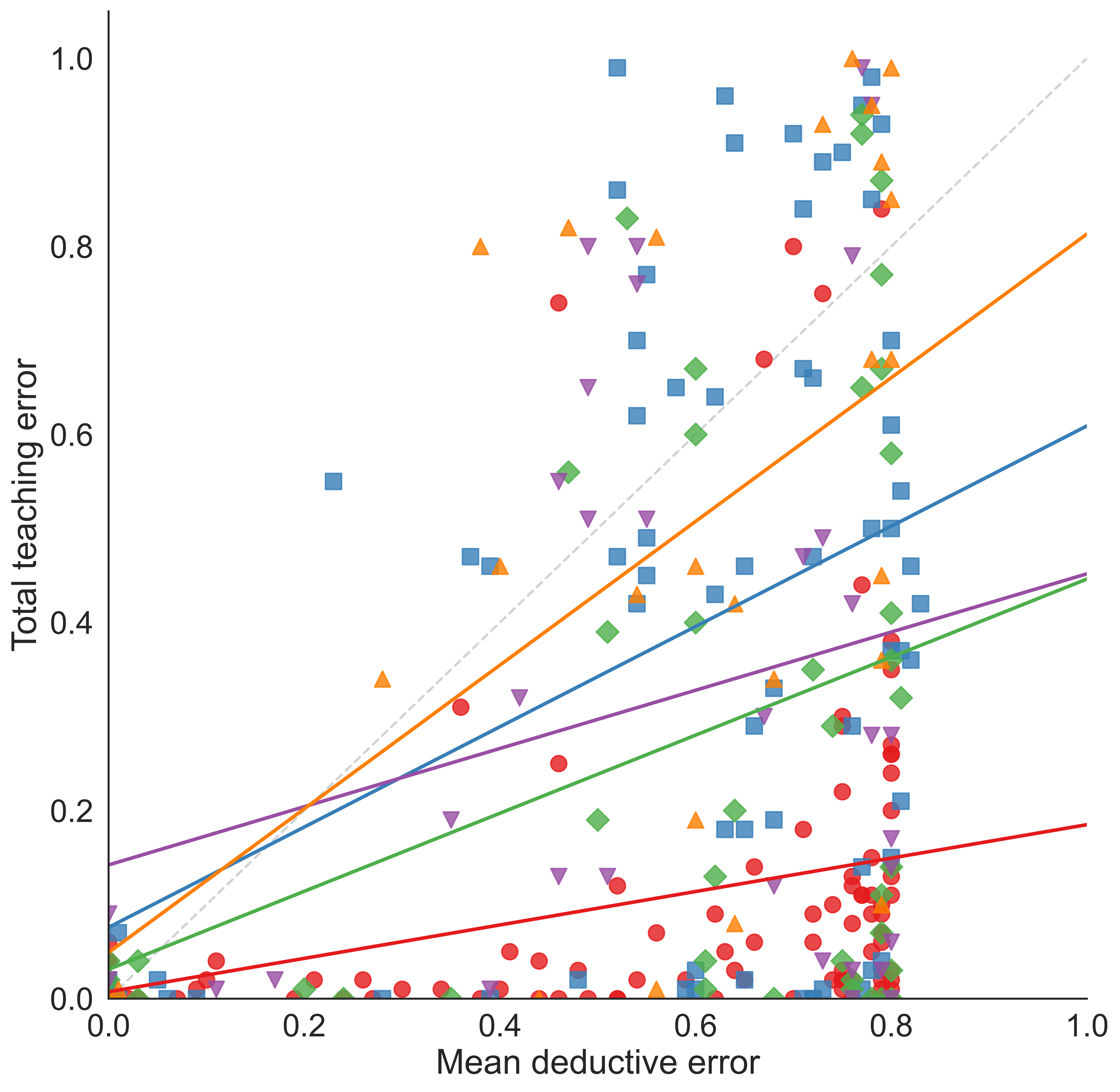} 
        \caption{Low total error. Teaching set of size 1.}
        \label{fig:llm_experiment_left}
    \end{subfigure}
    \hfill
    \begin{subfigure}[b]{0.45\textwidth}
        \centering
        \includegraphics[height=0.27\textheight]{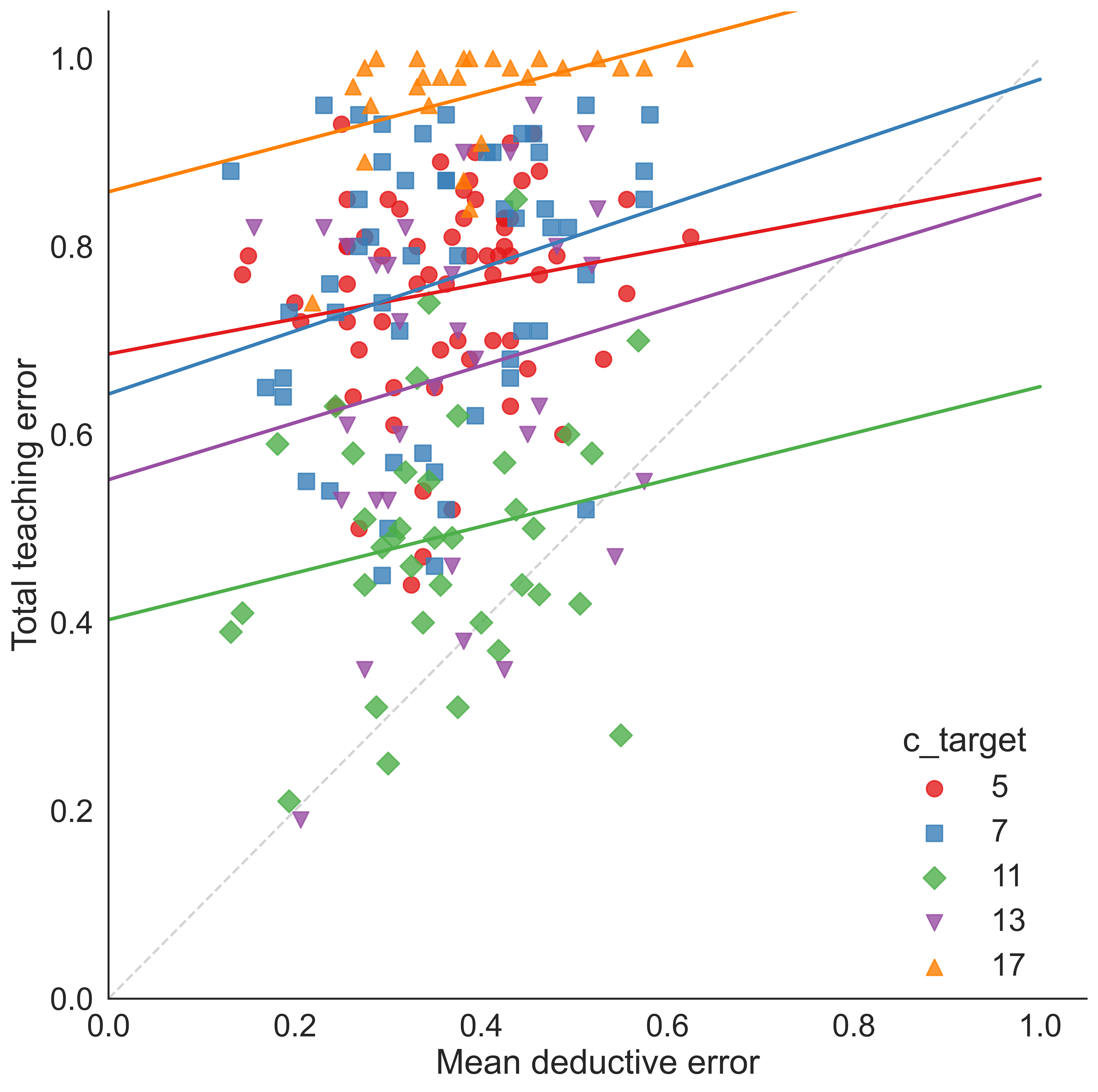} 
        \caption{High total error. Teaching set of size 2.}
        \label{fig:llm_experiment_right}
    \end{subfigure}
    \caption{
    Each data point is a teaching set, in (a) of size one: a single positive $x$, and in (b) of size two: one negative $x_1$ and one positive $x_2$. Total teaching error on the vertical axis. On the horizontal axis we have mean deductive error (averaged over all 5 values of k) for $x$ in (a) and averaged over $x_1$ and $x_2$ in (b). 
    Each linear fit shows a positive correlation, which is required for the PAC teaching framework to be applicable, and this is more pronounced for low total error in (a), as expected.}
    \label{fig:llm_experiment}
\end{figure}

Note the top left corner is empty, thus low deductive error predicts low total teaching error, and for each value of $k$ the increasing linear fit shows that deductive errors influence total teaching error, in line with the PAC teaching framework.
However, we also observe a cluster in the bottom-right corner, with high deductive error but low total teaching error. We attempt an explanation by example: checking if $7$ is a multiple of $735$ has a relatively high deductive error, but the teaching prompt specifies that only one of $7$ and $5$ is a divisor of $735$ so the LLM may confirm that $5$ is a divisor without even trying $7$, and hence $\gamma_L(7,735)$ does not impact the total teaching error. This hypothesis is supported by the fact that 5 has more points in this corner and also the lowest linear fit. 

While low total teaching error can occur even when mean deductive error is high, the key finding is that selecting examples with low deductive error reliably yields low teaching error. This experiment, in a scenario of low inductive effort, thus demonstrates a concrete case in which a teacher using teaching sets of low deductive errors can achieve a low total teaching error (often 0\% and always below 10\%) whereas a teacher oblivious to deductive error would risk a high total teaching error (often over 50\% and sometimes 100\%).  

We now move to a scenario with the same deductive errors but with higher total teaching error due to 
higher inductive load. We use the same concept class $C$ and example set $X$ but with teaching sets of size two, thus coming in pairs with one positive and one negative, with both examples needed to uniquely identify the target. Thus, note right away that there is no element in $X'$, used for teaching sets of size one, that will be used now.
To construct these pairs, we first select all positive examples that are consistent with at least two concepts. For each such example, we then pair it with a negative example that is consistent with the target concept but inconsistent with the remaining competing concepts, thereby eliminating them and ensuring that only the target is identified. As before, for each such pair we repeat 100 times, now with the following prompt: `Which of 5, 7, 11, 13 and 17 is a divisor of $x_1$, and not a divisor of $x_2$?', to extract total teaching error. See the plot in Figure \ref{fig:llm_experiment}(b) and compare it to \ref{fig:llm_experiment}(a). The results confirm our hypothesis: total error is higher, the outliers in the lower right quadrant have disappeared, and the correlation between total and deductive errors is still positive even though it has softened as the inductive error has become more prominent.

For other LLMs and for other teaching tasks, the derivation of deductive errors and total teaching error will have to change. The important finding is that these experiments support the hypothesis outlined at the start of this section, namely that deductive errors can be derived from LLMs as learners and that they are amenable to the PAC teaching framework.

\section{Learner and teacher behavior}\label{sec:LandT}

In this section we first specify two learner behaviors: a naive learner not aware of errors and a prudent PAC learner who is aware. Finally, three teacher behaviors are specified: a naive teacher not aware of errors, a heuristic PAC teacher and an optimal PAC teacher. 

Firstly, we may have a small or large concept class, in an extreme case even infinite, and we may also have a prior distribution over the concepts. Somehow, the learner will consider a finite subset of concepts, that we may call $\tilde{C}$, possibly $\tilde{C}=C$ for a small concept class, or $\tilde{C}$ is achieved by sampling from $C$ until it has a reasonable size containing a representative set of concepts including $c^*$. We will not go into details of how this is done, as it would vary based on the application.
From now on, we simply assume $\tilde{C}=C$, and note that this may be altered depending on the application.

The {\bf Naive learner} $L_N$ acts unaware of any errors being made in consistency checking. This learner initializes the set of possible guessed concepts to $P=C$, goes through each labelled example $(x,b) \in S$, that is received as a batch,  and discards any concept $c \in P$ that is not L-consistent, i.e. keeping $c$ only if $\hat{c}(x)=b$ \footnote{If  $\hat{c}=c, \forall c \in C$ this is the standard Version Space learner in Machine Teaching.}. When done we define $L_N(S)$ to be an arbitrary concept still left in $P$, unless $P$ is empty and $L_N$ reports failure. 

The {\bf Prudent PAC learner} $L_P$ is cognizant that errors can happen, so to avoid discarding the target it will for each concept $c \in C$ compute the count $Ct(c,S)$ of the number of labelled examples in $S$ that $c$ is L-consistent with, and arbitrarily guess a concept $L_P(S)$ with highest count.
In other words, if the max count is $k=max_{c \in C} Ct(c,S)$ the prudent learner selects a concept $L_P(S)$ arbitrarily from $\{c \in C: Ct(c,S)=k\}$. We make no assumptions about how this selection is made, and for a worst-case scenario we assume that if there is some concept that is not good among those with highest count it would have been chosen.




Let us now consider the teacher. 
The teacher knows the errors $\gamma_L$ made by the learner during consistency checking, and needs to find a teaching set $S$ making it probable that the concept $L_P(S)$ guessed by the prudent learner is approximately equal to the target $c^*$, either by $L_P(S)$ being $q$-similar to the target for identification teaching, or $L_P(S)$ being $q_L$-similar to the target for employment teaching. As is standard in Machine Teaching (and Machine Learning) the yardstick is the size of the teaching set, and for PAC Teaching this goal can be formulated in several ways. For fixed $q$ and $p$ the goal is to find a smallest teaching set $S$ satisfying Equation \ref{eq:PACTeaching}, whereas fixing $p$ and max-size $k$ the goal is to find a teaching set $S$ of size at most $k$ that for the highest possible $q$ satisfies Equation \ref{eq:PACTeaching}, and for fixed $q$ and $k$ the goal is to find a teaching set $S$ of size at most $k$ that for the highest possible $p$ satisfies Equation \ref{eq:PACTeaching}.

Note that we may have a small or large example set, 
in an extreme case even infinite, and we may also have a distribution over the examples. Somehow, the teacher will consider a finite subset of examples to include in the teaching set, that we could call $\tilde{X}$, possibly $\tilde{X}=X$ for a small example set or $\tilde{X}$ is achieved by sampling from $X$ until it has a reasonable size containing a representative set of examples.
We will not go into details of how this is done. 
From now on, we simply assume $\tilde{X}=X$, and note that this may be altered depending on the application.

Let us start with the {\bf Naive teacher} $T_N$ who acts as if there were no errors (as if $\gamma_L(c,x)=0$ always) and finds a teaching set $S$ labelled according to $c^*$ so that an error-free Version Space learner, who discards any concept not consistent, would be left only with $c^*$. 

A {\bf Heuristic PAC teacher} $T_H$ will on the other hand use some heuristic based on the errors in $\gamma_L$ and the behavior of $L_P$ for picking examples to include in $S$. In Section \ref{sec:heur} we discuss such heuristics.


Let us illustrate the benefits of PAC teaching with a very simple case highlighting the difference between the naive learner with a naive teacher versus the prudent learner with a heuristic teacher.
Consider two concepts and two examples where $c_1(x_1)=c_2(x_2)=1$, and $c_1(x_2)=c_2(x_1)=0$ with concept-independent errors $\gamma_L(c_1,x_1)=\gamma_L(c_2,x_1)=0.1$, and $\gamma_L(c_1,x_2)=\gamma_L(c_2,x_2)=0.2$. For target $c_2$ the naive teacher $T_N$ could use teaching set $S=\{(x_2,1)\}$ and get a probability of 0.64 that $L_P(S)=c_2$, i.e. that $c_2$ is the only concept with max L-consistency count. A heuristic PAC teacher $T_H$, noticing that $\gamma_L(c_2,x_1)<\gamma_L(c_2,x_2)$ would instead use $S=\{(x_1,0)\}$ for probability 0.81, and if requiring higher odds would use $S=\{(x_1,0),(x_2,1)\}$ for probability 0.8928. 

In this case errors are concept-independent, so on any correctly labelled example no concept has higher probability than the target concept of being $L$-consistent. Increasing the teaching set $S$ therefore increased the probability of $L_P(S)$ being the target concept. On the other hand, the naive learner $L_N$ will always be more likely to rule out the target concept $c^*$ when seeing a new labelled example $(x,c^*(x))$ with $\gamma_L(c^*,x)>0$. The above observation basically proves the following.

\begin{theorem}
If errors are concept-independent and below one-half, then adding more examples to the teaching set cannot decrease the chances of successful PAC teaching for the Prudent learner.
\end{theorem}

Finally, the {\bf Prudent-optimal PAC teacher} $T_P$ will find an optimal teaching set $S$ for $L_P$, depending on what is the optimization criterium: given $q$ and $p$ minimize size $k$ of $S$, or given $q$ and $k$ maximize the achievable $p$, or given $p$ and $k$ maximize the achievable $q$? 
Let us define this formally, both for identification (id) and for employment (em), for a total of six distinct optimization goals. 
Note each of these assumes the prudent learner $L_P$, but for ease of presentation we do not specify this in the notation used.

Firstly, let $0 \leq q,p \leq 1$ and $k$ a positive integer and say that $S$ is $(q,p,k)$-PAC-id if $|S| \leq k$ and $L_P(S)$ satisfies Equation \ref{eq:PACTeaching}. Likewise, say $(q,p,k)$-PAC-em when replacing $sim$ by $sim_L$.
Next,  say that $S$ is $(*,p,k)$-PAC-id optimal if it is $(q',p,k)$-PAC-id for some $q'$ and no $S'$ is $(q'',p,k)$-PAC-id for any $q'' < q'$.
Also, say that $S$ is $(q,*,k)$-PAC-id optimal if it is $(q,p',k)$-PAC-id for some $p'$ and no $S'$  is $(q,p'',k)$-PAC-id for any $p'' < p'$.
Also, say that $S$ is $(q,p,*)$-PAC-id optimal if it is $(q,p,|S|)$-PAC-id and no $S'$ is $(q,p,k)$-PAC-id for any $k<|S|$. 
 Analogously we define three PAC-em optimal sets. 
The optimal teacher $T_P$ thus comes in six versions, three for identification as defined below, and three analogous ones (probable-em-, approx-em-, and size-em-) for employment. 

\begin{itemize}
    \item the probable-id-optimizer takes as input $q,k$ and finds $S$ that is $(q,*,k)$-PAC-id optimal
    \item the size-id-optimizer takes as input $q,p$ and finds $S$ that is $(q,p,*)$-PAC-id optimal
    \item the approx-id-optimizer takes as input $p,k$ and finds $S$ that is $(*,p,k)$-PAC-id optimal
\end{itemize}





\section{Computing optimal and heuristic PAC Teaching sets}\label{sec:betteralgo}


In this section we consider various algorithms computing teaching sets, of both the prudent-optimal teacher and of the heuristic teacher.

\subsection{Algorithms for computing optimal teaching sets}

We start by investigating the computational complexity of the task faced by the prudent-optimal teacher $T_P$. We achieve tight runtime bounds for all six versions of this task: teaching for identification or employment while optimizing for size or approximation or probability. We give algorithms whose runtimes provide upper bounds matching the lower bounds given by standard complexity assumptions. 

The algorithms can be adapted for identification teaching or employment teaching simply by switching between similarity functions $sim$ and $sim_L$. 
 Firstly, note that in time $O(|C||X|)$, for a fixed approx-factor $q$ and target $c^*$, we compute the set of Good concepts $q$-similar (or $q_L$-similar) to $c^*$. 
The following Lemma for a given teaching set $S$ is very useful. 

\begin{lemma}\label{lem}
Given consistency matrix $C \times X$ and $\gamma_L$, teaching set $S$, target concept $c^*$ and approx factor $q$, and Good concepts $q$-similar (or $q_L$-similar) to $c^*$,
compute the probability that $L_P(S)$ is Good in time $O(|C|\cdot|S|^2)$.
\end{lemma}

\begin{proof}
In Algorithm 1 we give the pseudo code. Note that there are five outer for loops.

\begin{algorithm}[!htbp]
\caption{Calculate probability of Approximately Correct teaching for a given teaching set. Runtime $O(|C|\cdot|S|^2)$}\label{alg:p_of_ACT}
\begin{algorithmic}
\Require $S$, $C$, $\gamma_L$, $c^*$ and $q$, yielding a partition of concepts into $G$ good and $B$ bad that we also assume as input.
\State $k \gets |S|$
\For{$c \in C$} \Comment{$O(|C|\cdot |S|^2)$}
    \State $PMF_c = DP_{sub}(c,S)$ \Comment{ $O(|S|^2)$} 
    \For{$i\in \{0\cdots k\}$} \Comment{$O(|S|)$}
        \State $P[X_c = i]=PMF_c[i]$
    \EndFor 
\EndFor

\For{$c \in C$} \Comment{$O(|C|\cdot |S|)$}
    \State $P[X_c \leq 0]= P[X_c = 0]$
    \For{$i\in \{1\cdots k\}$} \Comment{$O(|S|)$}
        \State $P[X_c \leq i]=P[X_c=i]+P[X_c\leq i-1]$
    \EndFor 
\EndFor

\For{$i \in \{\ 0 \cdots k\}$} \Comment{$O(|G|\cdot|S|+|B|\cdot|S|)=O(|C|\cdot|S|)$}
    \State $P[max_{g\in G}X_g \leq i] = 1$
    \State $P[max_{b\in B}X_b \leq i] = 1$

    \For{$g\in G$} \Comment{$O(|G|)$}
        \State $P[max_{g\in G}X_g \leq i] \mathrel{*=} P[X_g \leq i]$
    \EndFor
    \For{$b\in B$} \Comment{$O(|B|)$}
        \State $P[max_{b\in B}X_b \leq i] \mathrel{*=} P[X_b \leq i]$
    \EndFor
\EndFor

\State $P[max_{g\in G}X_g = 0] = P[max_{g\in G}X_g \leq 0]$
\For{$i\in \{1\cdots k\}$} \Comment{$O(|S|)$}
    \State $P[max_{g\in G}X_g = i]= P[max_{g\in G}X_g \leq i]-P[max_{g\in G}X_g \leq i-1]$
\EndFor

\State $Probability = 0$
\For{$i \in \{1 \cdots k\}$} \Comment{$O(|S|)$}
    \State $Probability \mathrel{+=} P[max_{g\in G}X_g = i] \cdot P[max_{b\in B}X_b \leq i-1]$
\EndFor
\State \Return $Probability$
\end{algorithmic}
\end{algorithm}

\begin{algorithm}[!htbp]
\caption{$DP_{sub}(c,S)$. Find $DP[i][j]$, probability that $c$ is L-consistent $j$ times on the $i+1$ first examples of $S$, for some $0 \leq i \leq |S|-1, 0 \leq j \leq |S|$.}\label{alg:DP}
\begin{algorithmic}
\Require $c$, $S$, $\gamma_L$
\State $keep(x,c)=(1-\gamma_L(c,x)) \text{ if label of x is same as } c(x)$
\State $keep(x,c)=\gamma_L(c,x) \text{ if label of x is different from  } c(x)$
\State $k \gets |S|$
\State $DP \gets [k][k+1]$ \Comment{Initialization}
\State $DP[0][0] \gets 1- keep(S[0], c)$
\State $DP[0][1] \gets keep(S[0],c)$
\For{$j\in \{2\cdots k\}$}
    \State $DP[0][j] \gets 0$
\EndFor

\For{$i\in \{1\cdots k-1\}$}
    \State $DP[i][0] \gets DP[i-1][0]\cdot keep(S[i],c)$
\EndFor

\For{$i \in \{1\cdots k-1\}$} 
    \For{$j \in \{1\cdots k\}$}
        \State $DP[i][j] \gets DP[i-1][j-1]\cdot keep(S[i],c) + DP[i-1][j]\cdot (1-keep(S[i],c))$
    \EndFor
\EndFor
\For{$j \in \{0 \cdots k\}$}
\State $P[X_c = j] \gets DP[k-1][j]$
\EndFor
\end{algorithmic}
\end{algorithm}

Let us first argue for correctness.  We apply a top-down approach, arguing from the fifth and last outer for loop to the first one, and hence start with restating the goal. With $G$ the set of Good concepts and $B= C \setminus G$, we want to know the probability that a concept in $G$ is L-consistent more times than all concepts in $B$. We go over all the possible scenarios where at least one concept in $G$ is L-consistent at least once more than any concept in $B$, computed by the following formula

\begin{align*}
        P[max_{g \in G} X_g >& max_{b\in B} X_b | S] =\\
        &=\sum^{|S|}_{i=1}P[max_{g \in G} X_g = i | S] \cdot P[max_{b\in B} X_b \leq i-1 | S]
\end{align*}

In the above $X_c$ is simply shorthand for the count $Ct(c,S)$, and the reader can check that this formula corresponds to the fifth and last outer for loop of the pseudo-code. 

For sake of simplicity we omit the repeated $S$ from the rest of the equations. The fourth outer for loop computes the probability that the highest count over all concepts in C equals $i$, and the formula for this is given below. In the below equation we use a general set of concepts $S$, which in the pseudo-code is switched out with $G$ or $B$ depending on the calculation. 

\begin{equation*}
    P[max_{c \in S} X_c = i] = P[max_{c \in S} X_c \leq i] - P[max_{c \in S} X_c \leq i-1]
\end{equation*}

with the base case

\begin{equation*}
    P[max_{c \in S} X_c = 0] = P[max_{c \in S} X_c \leq 0] 
\end{equation*}

The third outer for loop computes the probability that the highest count is $\textit{at most}$ $i$ and the formula is given below

\begin{equation*}
    P[max_{c \in S} X_c \leq i] = \prod_{c\in S} P[X_c\leq i]
\end{equation*}

The second outer for loop calculates the probability of a concept $c$ having count at most $i$, and for this we use the Cumulative Distribution Function $CDF$

\begin{equation*}
    P[X_c \leq i] = P[X_c \leq i-1] + P[X_c= i]
\end{equation*}

with the base case

\begin{equation*}
    P[X_c \leq 0] = P[X_c= 0]
\end{equation*}

Finally the first outer for loop calculates the probability that a single concept $c$ has count exactly $i$, namely

\begin{equation*}
    P[X_c = i]\quad \forall_{i\in \{0\cdots |S|\}}
\end{equation*}

In our pseudo-code this is done by a straightforward Dynamic programming  called $DP_{sub}(c,S)$ given as a separate Algorithm \ref{alg:DP}. This DP fills a $|S|(|S|+1)$ size table where $DP[i][j]$, for some $0 \leq i \leq |S|-1, 0 \leq j \leq |S|$, holds the probability that $c$ is L-consistent $j$ times on the $i+1$ first examples.

As we have argued along the way, these formulas compute what we need and they correspond to the pseudo-code. We are therefore done with the correctness proof.

We now argue for the runtime. In the first outer for loop, the Dynamic Programming  called $DP_{sub}(c,S)$ has runtime $O(|S|^2)$, and since this is computed for each concept the first outer loop is $O(|C||S|^2)$.
The second and third outer for loops are $O(|C||S|)$, while the fourth and fifth are $O(|S|)$. We conclude that the runtime is $O(|C||S|^2)$
\end{proof}

This lemma is used as a subroutine  to give XP algorithms, parametrized by max-size $k$, for the probable-id-optimizer and  the probable-em-optimizer problems. We also show that the runtime of these algorithms cannot be improved to the point of removing $k$ from the exponent, unless the $W$-hierarchy collapses.

\begin{theorem}\label{thm:prob-opt}
The problem facing the Prudent-optimal PAC Teacher $T_P$ for probable-optimizing (id or em), given  consistency matrix $C \times X$ and $\gamma_L$, approx factor $q$ and max-size $k$, for a given target, can be solved in time $O(k^3 |C| |X|^k)$. Moreover, when parametrized by $k$ this problem is $W[2]$-hard.
\end{theorem}

\begin{proof}
    For given $q,k$ we search for the best probable-optimizing teaching set by going over all subsets of size at most $k$, and for each one run the algorithm of Lemma \ref{lem} to find the smallest value $p$ such that this set is $(q,p,k)$-PAC-id (alt. $(q,p,k)$-PAC-em). The set $S$ associated with the highest value $p$ achieved this way is a $(q,*,k)$-PAC-id (alt. $(q,*,k)$-PAC-em) optimal set. The total  runtime is $O( |C| \sum^k_{i=1} \binom{|X|}{i} \cdot i^2)$  which can be simplified to $O(k^3 |C| |X|^k)$.

    Improving this runtime to an FPT algorithm, i.e. to remove $k$ from the exponent, would imply a collapse of the W-hierarchy in Parameterized Complexity, as we now show. By a reduction similar to one given in \cite{shino} we show that the probable-optimizing problem is W-hard. In the $k$-Hitting Set problem we are given an integer $k$ and a family $F$ of subsets of a universe $U$ and asked if there are $k$ elements of the universe that hit every subset in $F$, and this problem is W[2]-hard \cite{Fellows}.
    The parameterized reduction from Hitting Set to the probable-optimizing problem constructs an instance with example set $X=U$, and for each $A\in F$ a concept $c_A \in C$ such that for $x \in X$ we set $c_A(x)=1$ if and only if $x \in A$. We add a target concept $c^*$ with $c^*(x)=0$ for all $x$, set $\gamma_L(c,x)=0$ for all $x,c$, set $q=1$, and crucially (for the parameterized aspect) use the value of $k$ as the size bound.
    Assuming there was an FPT algorithm for the probable-optimizer parameterized by $k$, then given an instance of $k$-Hitting Set we follow the above construction and run this algorithm on the constructed instance and check if the optimal set returned was $(q,1,k)$-PAC-id (or $(q,1,k)$-PAC-em) optimal, i.e. with $L(S)$ guessing a 1-similar concept with probability 1, in effect guessing the target always. If that is the case we have a yes-instance of the Hitting set problem, and otherwise we have a no-instance. This would constitute an FPT algorithm for $k$-Hitting Set and concludes the proof of the Theorem.
\end{proof}

By similar techniques, we are also able to give XP algorithms for the approx-id-optimizer and the approx-em-optimizer problems, parametrized by max-size $k$. Again, we show that the runtime cannot be improved to the point of removing $k$ from the exponent, unless the $W$-hierarchy collapses.

\begin{theorem}\label{thm:approx-opt}
The problem facing the Prudent-optimal PAC Teacher $T_P$ for approx-optimizing (id or em) within a fixed precision (say $d$ decimal digits), given  consistency matrix $C \times X$ and $\gamma_L$, prob factor $p$ and max-size $k$, for a given target, can be solved in time $O(k^3 |C| |X|^k)$. Moreover, when parametrized by $k$ this problem is $W[2]$-hard.
\end{theorem}

\begin{proof}
    For given $p,k$ we search for the best approx-optimizing teaching set by using an outer loop that performs a binary search for the best value $q$ within $d$ decimal digits. This amounts to at most $log_2 (10^d)$ iterations of the outer loop, which is $O(1)$ for constant $d$. For each such value of $q$ we go over subsets of size at most $k$, and  run the algorithm of Lemma \ref{lem}, stopping as soon as the output probability  $p'$ is at least $p$, meaning this set is $(q,p,k)$-PAC-id (alt. $(q,p,k)$-PAC-em). When stopping this way we continue the binary search in the larger range, whereas if no set of size at most $k$ gives the successful stop then we continue in the lower range. For the highest successful value of $q$ the associated set $S$ that forced the stop is a $(*,p,k)$-PAC-id (alt. $(*,p,k)$-PAC-em) optimal set. The total  runtime is $O( |C| \sum^k_{i=1} \binom{|X|}{i} \cdot i^2)$  which can be simplified to $O(k^3 |C| |X|^k)$.

    Improving this runtime to an FPT algorithm, i.e. to remove $k$ from the exponent, would imply a collapse of the W-hierarchy in Parameterized Complexity, by a similar reduction as the one given above in the proof of Theorem \ref{thm:prob-opt}. The difference to that reduction is only that now we set $p=1$ and do the reduction to the case of fixed precision on a single decimal digit. Then, we have  a yes-instance of $k$-Hitting Set if and only if the highest successful value returned is $q=1.0$.
   Assuming there was an FPT algorithm for the approx-optimizer parameterized by $k$, there would also be one for $k$-Hitting Set. This concludes the proof of the Theorem.
\end{proof}

We also give algorithms for the size-id-optimizer and size-em-optimizer problems. These are again shown to have runtimes that are asymptotically optimal under standard complexity assumptions, in this case
the  Exponential Time Hypothesis and the Strong Exponential Time Hypothesis.

\begin{theorem}\label{thm:size-opt}
The problem facing the Prudent-optimal PAC Teacher $T_P$ for size-optimizing (id or em), given  consistency matrix $C \times X$ with $|X|=m$ and $\gamma_L$, approx factor $q$ and prob factor $p$, for a given target, can be solved in time $O(m|C|2^m)$. Moreover, under ETH no algorithm with runtime $2^{o(m)}$ exists, and under Strong ETH no algorithm with runtime $(2-\epsilon)^{m}$ exists for any $\epsilon > 0$. 
\end{theorem}

\begin{proof}
    For given $q,p$ we search for the size-optimizing teaching set by going over all subsets and for each such $S$ run the algorithm of Lemma \ref{lem} to check if $S$ is $(q,p,|S|)$-PAC-id (alt. $(q,p,|S|)$-PAC-em). The largest value $k$ with a set $S$ with $|S|=k$ giving a positive answer means a $(q,p,*)$-PAC-id (or $(q,p,*)$-PAC-em) optimal set of size $k$. The total  runtime is $O( m|C| 2^m)$.

    It is known that the Hitting Set problem on a universe of size $m$ cannot be solved in time $2^{o(m)}$ unless ETH fails \cite{calabro}, and under Strong ETH no algorithm with runtime $(2-\epsilon)^{m}$ for any $\epsilon > 0$ \cite{cygan}. We reduce Hitting Set to size-optimizing teacher by giving a similar reduction as the one given above in the proof of Theorem \ref{thm:prob-opt}, to show the lower bound stated in the Theorem. The difference to that reduction is only that now we fix $q=p=1$ and the size-constraint varies with the size of $S$. Moreover, the answer is checked against the size of the smallest $S$ returned and the Hitting Set is a yes-instance if and only if $|S| \leq k$, for the value of $k$ given as part of the Hitting Set instance.
    Assuming there was an algorithm for the size-optimizer with runtime $O(2^{o(|X|)})$ then this would translate, through this reduction, to an algorithm solving Hitting Set in time $2^{o(m)}$, refuting ETH.
Likewise, an algorithm for the size-optimizer with runtime $O^*((2-\epsilon)^{|X|})$ for $\epsilon >0$  would translate, through this reduction, to an algorithm solving Hitting Set in time $O^*((2-\epsilon)^{m})$, refuting Strong ETH.
    This concludes the proof of the Theorem.
\end{proof}


\subsection{Heuristic algorithms computing teaching sets}\label{sec:heur}

We discuss algorithms for $T_H$, i.e. heuristics computing good teaching sets for an error-aware learner like $L_P$.
One option is to focus on the \textbf{Probably} Correct aspect of PAC Teaching. Intuitively, we want to teach with examples that {\em uniquely identify} the target, at least with few concepts sharing the label of the target on this example.  For a target concept $c^*$  first go over each example $x\in X$ and calculate, for some distribution
over concepts, the expected value of the event that c is L-consistent with $(x,c^*(x))$. We call this the Uniqueness of $x$.
\begin{equation*}
  Uniqueness(x)  
  = \mathbb{E}_{c \sim {\mathcal D}(C)} [ P(\hat{c}(x) = c^*(x))] 
    \label{eq:uniqueness}
\end{equation*}
Then sort the examples by the Uniqueness value and construct a teaching set greedily adding examples of lowest value first. 

Another option is to focus on the \textbf{Approximately} Correct aspect of PAC Teaching. Intuitively, we want to teach with an example $x'$ such that for the set of concepts $C_{x'}$ that share the label
of the target on this example, it holds that these concepts are similar to the target.  For a target concept $c^*$  first go over each example $x' \in X$ and calculate the set of concepts $C_{x'} = \{c \in C: c(x')=c^*(x')\}$, and then for some distribution
over examples, calculate the expected value of the event that $c$ drawn from $C_{x'}$, by some distribution,  is L-consistent with $(x,c^*(x))$.  We call this the Homogeneity-value of $x'$. Then sort the examples according to Homogeneity-value and construct a teaching set greedily adding examples of highest value first. 

We implemented and ran these two heuristics on the concept class over multiples-of-k given in the plots of Figure \ref{fig:llm_experiment}. In the below table we see the results for the case of a teaching set of size 1. Both heuristics select a teaching set $x$ of very low total teaching error of at most 0.02, for all 5 values of k, showing their usefulness in this case. 
\begin{table}[ht]
\centering\caption{Selected teaching example `x' (teaching size 1) by heuristic algorithm. The selected teaching examples have consistently low teaching error across all target concepts. }
\begin{tabular}{clcccc}
\hline
$c_{\text{target}}$ & Criterion & $x$ & Heuristic Score & Mean Deductive Error & Teaching Error \\
\hline
\multirow{2}{*}{5}
    & Uniqueness  & 10  & 0.0 & 0.00 & 0.02 \\
    & Homogeneity & 10 & 0.124 & 0.00 & 0.02 \\
\hline
\multirow{2}{*}{7}
    & Uniqueness  & 14  & 0.0 & 0.00 & 0.00 \\
    & Homogeneity & 14 & 0.089 & 0.00 & 0.00 \\
\hline
\multirow{2}{*}{11}
    & Uniqueness  & 11  & 0.0 & 0.00 & 0.02 \\
    & Homogeneity & 305 & 0.086 & 0.79 & 0.00 \\
\hline
\multirow{2}{*}{13}
    & Uniqueness  & 13  & 0.0 & 0.00 & 0.00 \\
    & Homogeneity & 132 & 0.073 & 0.76 & 0.02 \\
\hline
\multirow{2}{*}{17}
    & Uniqueness  & 34  & 0.0 & 0.00 & 0.00 \\
    & Homogeneity & 34 & 0.077 & 0.00 & 0.00 \\
\hline
\end{tabular}
\label{tab:h_algo_1td}
\end{table}

When deductive error is less influential, see Table \ref{tab:h_algo_2td} for teaching sets of size two, these heuristics should be less useful.

\begin{table}[ht]
\centering
\caption{Selected teaching set `$x_1, x_2$' by heuristic algorithm in the scenario with high total teaching error. As deductive errors are less correlated with teaching error in this setting, the non-optimal teaching error resulting from these PAC teaching sets is expected. }
\begin{tabular}{clcccc}
\hline
$c_{\text{target}}$ & Criterion & $x_1$, $x_2$ & Heuristic Score & Mean Deductive Error & Teaching Error \\
\hline
\multirow{2}{*}{5}
    & Uniqueness  & \{35, 119\}  & 0.13 & 0.23 & 0.77 \\
    & Homogeneity & (35, 119) & 0.13 & 0.23 & 0.77 \\
\hline
\multirow{2}{*}{7}
    & Uniqueness  & \{35, 660\}  & 0.2 & 0.30 & 0.64 \\
    & Homogeneity & \{455, 325\} & 0.2 & 0.43 & 0.94 \\
\hline
\multirow{2}{*}{11}
    & Uniqueness  & \{55, 190\}  & 0.3 & 0.18 & 0.39 \\
    & Homogeneity & \{55, 190\} & 0.3 & 0.18 & 0.39 \\
\hline
\multirow{2}{*}{13}
    & Uniqueness  & \{429, 374\}  & 0.38 & 0.44 & 0.78 \\
    & Homogeneity & \{845, 750\} & 0.38 & 0.38 & 0.65 \\
\hline
\multirow{2}{*}{17}
    & Uniqueness  & \{357, 672\}  & 0.28 & 0.50 & 1.00 \\
    & Homogeneity & \{595, 385\} & 0.28 & 0.48 & 0.98 \\
\hline
\end{tabular}
\label{tab:h_algo_2td}
\end{table}

We can also combine to give a heuristic for the \textbf{Probably and Approximately} Correct aspect of PAC Teaching. Calculate a combined value for each example, based on some weighting between having low Uniqueness and simultaneously having high Homogeneity, say $Combined(x)= (1-Uniqueness(x))+ \alpha \times Homogeneity(x)$ for a chosen factor $\alpha$. Then sort by these values and construct a teaching set greedily adding examples of highest $Combined$ value first.

\section{Conclusion}

We have introduced the PAC teaching framework, that is useful whenever we have a learner who does not perform perfect deductive inference. We have defined different behaviors of both the teacher and the learner in situations where these errors are stochastic and measurable. We have given algorithms finding the optimal teaching sets in these situations, and shown that the runtime of these algorithms cannot be improved unless standard complexity assumptions fail. 

The PAC teaching framework is useful even if we do not know the algorithms used by the learner. This is clear from Table \ref{tab:h_algo_1td}, where 
the LLM is not applying a learning algorithm known to us, but after deriving its deductive errors we are able to use a heuristic algorithm to compute teaching sets that increase the chance of successful PAC teaching.


\subsection*{Acknowledgements}

JHO thanks CIPROM/2022/6 (FASSLOW)  funded by Generalitat Valenciana, and Spanish grant PID2024-162030OB-100 (ROBIN) funded by MCIN/AEI/10.13039/501100011033 and ERDF A way of making Europe, Cátedra ENIA-UPV in Sustainable AI Development, TSI-100930-2023-9, and INCIBE’s Chair funded by the EU-NextGenerationEU through the Spanish government's Plan de Recuperación, Transformación y Resiliencia, and EUR2024-153548 (PREDAIT) ``Towards Predictable AI" from ``Spanish Europe Excelencia" 2024. JHO's research is supported by OpenAI's grant to the 'AI Progress through the Lens of Predictable AI Ecosystems' programme, which is based at the Leverhulme Centre for the Future of Intelligence at the University of Cambridge. 

\bibliography{references}

\newpage

\section{Appendix}


\subsection{LLM prompts}
Here are the prompts used in our LLM experiments. See also the Github repository mentioned in the footnote.

\begin{systemprompt}[Deductive Error - System prompt]
    Answer format: 'Answer = Yes' or 'Answer = No'
\end{systemprompt}
\begin{systemprompt}[Deductive Error - User message]
    Is \{k\} a divisor of \{x\}?
\end{systemprompt}

\begin{systemprompt}[Low Inductive Load - Teaching System prompt]
    Answer format: Answer = 5 or Answer = 7 or Answer = 11 or Answer = 13 or Answer = 17
\end{systemprompt}

\begin{systemprompt}[Low Inductive Load - Teaching User message]
    Which of 5, 7, 11, 13 and 17 is a divisor of \{x\}?
\end{systemprompt}

\begin{systemprompt}[High Inductive load - Teaching System prompt ]
Answer format: Answer = 5 or Answer = 7 or Answer = 11 or Answer = 13 or Answer = 17
\end{systemprompt}

\begin{systemprompt}[High Inductive load - Teaching User message]
    Which of 5, 7, 11, 13 and 17 is not a divisor of \{x1\}, while it is a divisor of \{x2\}?
\end{systemprompt}

\end{document}